\documentclass[conference]{IEEEtran}
\IEEEoverridecommandlockouts

\usepackage{cite}
\usepackage{amsmath,amssymb,amsfonts}
\usepackage{amsthm}
\usepackage{booktabs}
\usepackage{multirow}
\usepackage{tabularx}
\usepackage{graphicx}
\usepackage{xcolor}
\usepackage{url}
\usepackage{float}
\usepackage[ruled,vlined,linesnumbered]{algorithm2e}
\usepackage{etoolbox}

\AtBeginEnvironment{thebibliography}{\enlargethispage{6pt}\footnotesize\setlength{\itemsep}{-0.6pt}\setlength{\parskip}{0pt}}
\SetKwInput{KwIn}{Input}
\SetKwInput{KwOut}{Output}
\SetKwProg{Fn}{Function}{:}{}
\SetKw{KwBreak}{break}
\newcolumntype{Y}{>{\raggedright\arraybackslash}X}
\newtheorem{proposition}{Proposition}
\newcommand{\kw}[1]{\textsc{#1}}

\newcommand{\R}{\mathbb{R}}
\newcommand{\G}{\mathcal{G}}
\newcommand{\V}{\mathcal{V}}
\newcommand{\E}{\mathcal{E}}
\newcommand{\M}{\mathcal{M}}
\newcommand{\Lcal}{\mathcal{L}}
\newcommand{\X}{\boldsymbol{X}}
\newcommand{\Z}{\boldsymbol{Z}}
\newcommand{\Hh}{\boldsymbol{H}}
\newcommand{\A}{\boldsymbol{A}}
\newcommand{\D}{\boldsymbol{D}}
\newcommand{\I}{\boldsymbol{I}}
\newcommand{\Lm}{\boldsymbol{L}}
\newcommand{\Lb}{\bar{\boldsymbol{L}}}
\newcommand{\Ut}{\boldsymbol{U}}
\newcommand{\Lambdab}{\boldsymbol{\Lambda}}

\newcommand{\softmax}{\operatorname{softmax}}
\newcommand{\tr}{\operatorname{tr}}
\def\BibTeX{{\rm B\kern-.05em{\sc i\kern-.025em b}\kern-.08em
    T\kern-.1667em\lower.7ex\hbox{E}\kern-.125emX}}

\title{SMGFM: Spectral Multimodal Graph Pretraining for Multimodal-Attributed Graphs}

\author{Zhengyu Wu, Xu Wang, Hongchao Qin, Xunkai Li, Guang Zeng, Rong-Hua Li, Guoren Wang}

\begin{document}

\maketitle

\begin{abstract}

Multimodal-attributed graphs (MAGs) couple graph topology with node semantics from text, images, and other modalities, making them important for real-world applications. 
Traditional graph learning contextualizes node semantics by coupling topology with node features. 
However, this coupling design becomes troublesome in MAGs, where structure-induced and modality-intrinsic semantics may contribute differently to downstream tasks. Structure-induced semantics promote relational consistency through smooth topological variation, whereas modality-intrinsic semantics often encode local, fine-grained distinctions that should not be uniformly smoothed or aligned. Existing multimodal graph encoders typically fuse topology and modalities in a shared spatial representation space, delaying semantic role identification until after interaction and thereby entangling semantics that require different treatments. 
Therefore, the key challenge is to identify semantic roles before cross-modal fusion. 
To this end, we leverage graph-frequency variation as a prior, where low-frequency components capture topology-consistent semantics and high-frequency components preserve modality-specific semantics. 
Based on this intuition, we propose \textit{SMGFM}, a spectral multimodal graph pretraining framework that decomposes each modality-specific node signal into graph-frequency bands and assigns band-level semantic roles before cross-modal interaction. 
Concretely, SMGFM constructs frequency-resolved modality tokens with scalable Chebyshev filters, estimates their coupling reliability through topology-conditioned routing, and performs band-modality interaction before fusion. 
Its frequency-routed objectives align smooth consensus routes while preserving modality-specific routes, mitigating spatial-domain entanglement and uniform cross-modal alignment.
Extensive experiments conducted on the MAG datasets demonstrate that SMGFM achieves state-of-the-art performance across graph-level and modality-level tasks.
\end{abstract}


\section{Introduction}
Multimodal-Attributed Graphs (MAGs)~\cite{wei2019mmgcn,tao2020mgat,LGMRec,MIG-GT} couple graph topology with node attributes from text, images, and other modalities, and have attracted growing attention in recent years for information retrieval, recommendation, social analysis, and scientific discovery~\cite{ektefaie2023multimodal,zhu2025mmgraph,cellur,wan2026openmag}.
Across these domains, topology and multimodal attributes provide complementary but non-equivalent sources of information.
On the one hand, topology implicitly encodes relational interactions among nodes; on the other hand, multimodal attributes offer diversified node-intrinsic semantics.
This dual evidence source makes MAGs a natural testbed for reusable graph pretraining, where Graph Foundation Models (GFMs) aim to support transferable graph representations and recent Multimodal Graph Foundation Models (MGFMs) further seek unified interfaces across topology, modalities, and tasks~\cite{GFM_Benchmark,gfm_unigraph,gfm_graphclip,gfm_gft,he2025unigraph2,liu2026planet}.
However, foundation-style MAG pretraining is more difficult than attaching a graph encoder to a shared multimodal space, because the same graph signal may provide useful relational consensus for one task while disrupting modality-specific evidence for another.
Traditional graph learning methods commonly couple topology with node features through neighborhood aggregation, so that node semantics are contextualized by relational structures.
In the context of reusable MAG encoders, directly coupling topology with all modality features can obscure whether the learned representation reflects structure-induced semantics or modality-intrinsic semantics.

Despite their notable advances, existing MAG studies expose a central challenge in semantic coupling, since structure-induced semantics and modality-intrinsic semantics are not equally useful across downstream objectives.
Structure-induced smoothness can propagate relational and higher-order semantics across neighborhoods, but excessive smoothing may suppress modality-specific details and cross-modal discrepancies needed for fine-grained discrimination.
Conversely, preserving modality-intrinsic semantics without structural coordination may benefit retrieval-sensitive representations, yet it can weaken the relational consistency required for graph-level inference.
For foundation-style pretraining, this tension is amplified: a unified representation must remain useful across graph-centric and modality-centric tasks, but a uniform coupling rule may induce negative transfer between smooth structural consensus and private modality evidence.
Therefore, the key difficulty is not merely how to combine topology and modalities, but how to decide which semantics should be coupled, preserved, or separated before representation learning becomes irreversible.
\textbf{L1. Spatial-domain entanglement.}
Most multimodal graph models instantiate semantic coupling through operations performed directly on the graph structure, causing the role of each semantic component to be determined only after topology and modalities have already interacted~\cite{jia2023mhgat,graph4mm,GraphGPT-O, 2020h2gcn}.
This is especially restrictive for reusable encoders, because downstream heads can adapt only to the entangled representation they receive.
\textbf{L2. Uniform cross-modal alignment.}
Optimizing different modalities toward a shared representation objective may over-align modality-private information and obscure semantic distinctions that should remain modality-specific~\cite{you2020grace,GCC,hou2023graphmae2,gfm_unigraph,he2025unigraph2}.

Taken together, these limitations show that the key bottleneck of MAG pretraining lies in when semantic roles are identified. Once structure-induced semantics and modality-intrinsic semantics have been fused into a spatial representation, downstream encoders can only adapt to an already entangled embedding rather than recover the semantic distinctions required by different tasks. This motivates the research question studied in this paper, \textit{how can we pretrain a reusable MAG encoder that separates graph-smooth semantic consensus from modality-intrinsic semantics before cross-modal interaction?} Addressing this question requires the encoder to first expose how each modality varies over the graph, and then apply role-dependent objectives so that shared smooth semantics are aligned while modality-specific semantics remain discriminative. To this end, we propose \underline{\textbf{SMGFM}}, a spectral multimodal graph pretraining framework for MAGs. Our key insight is that graph frequency provides a pre-interaction prior for foundation-style MAG pretraining: (1) smooth frequency components are more likely to encode topology-consistent semantic consensus, while (2) rapidly varying components often contain modality-intrinsic semantics that should not be suppressed.
SMGFM implements this idea by using scalable Chebyshev filters~\cite{defferrard2016chebnet} to transform each modality into band-specific tokens, so semantic coupling is performed at a frequency-resolved level rather than in a single spatial representation.
A topology-conditioned router then evaluates token reliability from graph smoothness, spectral energy, cross-modal agreement, and modality availability, allowing role-constrained interaction to exchange semantics only when the corresponding frequency bands are reliable for coupling.
During pretraining, SMGFM aligns smooth consensus routes while preserving modality-specific routes, thereby avoiding the uniform alignment pressure that may erase discriminative modality semantics.
Experiments provide evidence from controlled text-image binding, matched-protocol real-data graph tasks, spectral diagnostics, ablations, sensitivity analysis, and boundary tests under missing modalities, heterophily, and efficiency constraints.

\noindent \textbf{Our Contributions are as follows:} \\
\textbf{(1) Foundation-aware Perspective:}
We identify semantic role identification as a missing prerequisite for reusable MAG pretraining, showing that structure-induced semantics and modality-intrinsic semantics should not be uniformly coupled.
This perspective connects MGFM motivation with frequency-aware treatment of graph-smooth consensus and modality-private evidence.\\
\textbf{(2) New Paradigm:}
We propose \underline{\textbf{SMGFM}}, a spectral multimodal graph pretraining paradigm that decomposes modality semantics into graph-frequency bands and performs topology-conditioned role routing before interaction.
SMGFM aligns smooth semantic consensus while preserving modality-specific semantics through frequency-routed objectives.\\
\textbf{(3) Evidence-backed Validation:}
Experiments demonstrate controlled cross-modal binding and matched-protocol graph-task utility under fixed experimental settings.
Ablation studies, spectral diagnostics, sensitivity analysis, and boundary checks further verify the effectiveness and limits of frequency-aware semantic role modeling.

\section{Related Work}
\subsection{Multimodal Graph Learning and Pretraining}
Multimodal graph learning combines topology with text, images, audio, tabular attributes, and other node evidence.
Task-oriented models usually project modalities into a shared space and then apply graph convolution, graph attention, modality attention, or local-global graph modeling for recommendation, classification, retrieval, and clustering \cite{wei2019mmgcn,tao2020mgat,jia2023mhgat,LGMRec,MIG-GT}.
Recent MAG benchmarks broaden this view by evaluating multiple tasks and modalities under common protocols \cite{zhu2025mmgraph,yan2024multimodalbenchmark,wan2026openmag}.
Graph self-supervised learning, graph foundation models, and multimodal graph pretraining further pursue reusable encoders through contrastive objectives, masked reconstruction, unified embedding spaces, and topology-aware semantic binding \cite{you2020grace,gfm_gft,he2025unigraph2,liu2026planet,graph4mm,GraphGPT-O}.
These works motivate reusable MAG encoders, but most apply fusion or alignment after topology and modalities have already been mixed in a spatial representation.
Consequently, they mainly address how to combine graph and modality evidence after a shared representation exists, while Limitation One asks whether that representation has already entangled graph-smooth consensus with modality-private cues.
SMGFM changes the interaction point by decomposing each modality into graph-frequency bands before alignment and fusion.

\subsection{Spectral Graph Learning}
Spectral graph learning separates smooth and rapidly varying graph signals through the Laplacian eigenbasis and Dirichlet energy \cite{shuman2013emerging,vonluxburg2007tutorial}.
ChebNet made this view practical by replacing eigendecomposition with localized Chebyshev filters \cite{defferrard2016chebnet}, while propagation and polynomial models simplify or deepen frequency-sensitive filtering \cite{wu2019sgc,2019appnp,chen2020gcnii,zhu2021ssgc,zhang2022nafs}.
Adaptive spectral models then learn flexible responses for homophily, heterophily, expressive filtering, and graph transformers \cite{chien2021gprgnn,he2021bernnet,dong2021adagnn,pmlr2022Jacobigcn,guo2023optimal_poly_gnn,zheng2023_TPAMI_3,bo2023specformer,kreuzer2021san,rampasek2022graphgps}.
Heterophily-oriented designs show that useful graph evidence is not always smooth \cite{2020h2gcn,du2022gbkgnn,song2023ordergnn}.
These methods make non-smooth graph evidence visible, but they mainly study single-view features or representation-level spectral responses.
They therefore do not specify whether a high-frequency text or image cue should be treated as conflict, complementary private evidence, or noise in a MAG encoder.
In MAGs, a high-frequency component may be useful modality-private evidence, so SMGFM uses spectral filtering as the first stage of multimodal routing.

\subsection{Cross-modal Alignment}
Cross-modal alignment learns compatible representations across modalities and is central to modern text-image representation learning \cite{radford2021clip,imagebind,gsmn}.
Graph contrastive learning similarly uses positive and negative views to improve transferable graph encoders \cite{you2020grace,GCC,liu2022spgcl,mo2022sugrl}.
For MAGs, uniform alignment may over-smooth or over-align private node evidence.
SMGFM therefore aligns low-frequency consensus while preserving high-frequency modality-private evidence, treating compatibility as a role-specific objective rather than a uniform fusion rule.

\newpage
\section{Preliminaries}
\label{sec:preliminaries}
This section defines the notation, graph-frequency operators, and task interfaces used by SMGFM.
The goal is to make the later frequency-first tokenization, topology-conditioned routing, and frequency-routed objectives precise without introducing additional assumptions.

\begin{table}[H]
\centering
\caption{Notation used throughout Preliminaries and Method.}
\label{tab:notation}
\resizebox{\linewidth}{!}{
\begin{tabular}{ll}
\toprule
Symbol & Meaning \\
\midrule
$\G,\V,\E,\A,\D$ & MAG, nodes, edges, adjacency, and degree matrices \\
$\tilde{\A},\I,\Lm,\Lb$ & normalized adjacency, identity, Laplacian, and scaled Laplacian \\
$\X^{(m)},\Hh_0^{(m)}$ & encoded and projected signal of modality $m$ \\
$\Hh_b^{(m)},\tilde{\Hh}_b^{(m)}$ & pre- and post-mixer band-$b$ modality-$m$ tokens \\
$\Z_i,\Z_i^{(m)},\Z_e$ & fused node, modality-specific node, and edge embeddings \\
$\mathcal{B}_C,\mathcal{B}_P,\mathcal{B}_U$ & consensus, private, and uncertain band sets \\
$\rho_i^{(m)},r_{i,b}^{(m)},\omega$ & modality availability, routing score, and route weight \\
\bottomrule
\end{tabular}}
\end{table}

\subsection{Problem Formulation}
Consider a Multimodal-Attributed Graph (MAG) $\G=(\V,\E,\A)$, where $\V$ is the node set, $\E$ is the edge set, $\A\in\R^{n\times n}$ is the adjacency matrix, $n=|\V|$, and $\D$ is the degree matrix.
Each node $v_i$ can contain raw modality attributes $\{r_i^{(m)}:m\in\M\}$, where $\M$ may include text, image, audio, tabular, or other modalities.
After modality-specific encoding, these raw attributes become modality graph signals $\X^{(m)}\in\R^{n\times d_m}$; optional edge attributes $\X^{(e)}\in\R^{|\E|\times d_e}$ can support edge reconstruction or link prediction.
For supervised node tasks, $\boldsymbol{Y}\in\R^{n\times C}$ denotes labels over $C$ classes, with labeled, validation, and test subsets fixed before adaptation.
For incomplete MAGs, $\rho_i^{(m)}\in[0,1]$ records whether modality $m$ is observed and usable for node $v_i$.
The pretrained encoder is denoted by $f_\theta(\G,\{\X^{(m)}\}_{m\in\M},\X^{(e)})$ and returns fused node embeddings $\Z_i$, modality-specific node embeddings $\Z_i^{(m)}$, and optional edge embeddings $\Z_e$.

\subsection{Graph-Frequency Operators and Band Roles}
SMGFM treats each modality matrix as a graph signal and measures frequency by edge-wise variation.
The topology operators are the normalized adjacency and its Laplacian:
\begin{equation}
\tilde{\A}=\D^{-1/2}\A\D^{-1/2},\quad
\Lm=\I-\tilde{\A}.
\label{eq:prelim-operators}
\end{equation}
Here $\tilde{\A}$ supports degree-balanced propagation, while $\Lm$ measures departures from graph smoothness \cite{shuman2013emerging,vonluxburg2007tutorial}.

Signal variation is quantified by Dirichlet energy:
\[
\mathcal{E}_{\Lm}(\X)=\tr(\X^\top\Lm\X).
\]
A low-frequency signal has small $\mathcal{E}_{\Lm}$ and changes smoothly along edges; a high-frequency signal has larger energy and changes sharply across connected nodes, motivating the separation between consensus and modality-private evidence.

For scalable Chebyshev filtering, SMGFM rescales the Laplacian as
\[
\Lb=\frac{2}{\lambda_{\max}}\Lm-\I,
\]
placing the spectrum in the Chebyshev operating range.
Rather than using the exact but costly eigendecomposition $\Lm=\Ut\Lambdab\Ut^\top$, polynomial filters approximate graph-frequency bands without explicitly computing eigenvectors \cite{hammond2011wavelets,defferrard2016chebnet,wu2019sgc,zhu2021ssgc,pmlr2022Jacobigcn,guo2023optimal_poly_gnn}.

The $b$-th band filter is written as
\[
g_b(\Lm)=\sum_{k=0}^{K}\alpha_{b,k}T_k(\Lb),
\]
where $T_k$ is the Chebyshev basis and $\alpha_{b,k}$ is a learnable coefficient.
Applying $g_b(\Lm)$ to $\X^{(m)}$ produces band tokens $\Hh_b^{(m)}$, the evidence units later used for routing, interaction, fusion, and frequency-specific losses.
We use $\mathcal{B}_C$, $\mathcal{B}_P$, and $\mathcal{B}_U$ for consensus, private, and uncertain band sets: low bands provide graph-smooth consensus candidates, high bands preserve modality-private evidence, and uncertain bands are assigned by topology-conditioned routing.

\subsection{Pretraining and Evaluation Interfaces}
SMGFM separates reusable pretraining outputs from the lightweight heads used for evaluation.
Graph-centric interfaces use fused graph-aware representations: node classification applies a linear or MLP head to $\Z_i$, link prediction scores candidate edges from endpoint embeddings and $\Z_e$ when available, node clustering groups $\Z_i$ without label supervision, and heterophily stress tests probe behavior under edge conflict.
Modality-centric interfaces use modality-specific or paired representations: text-image retrieval compares $\Z_i^{(m)}$ across paired modalities, and missing-modality stress tests vary $\rho_i^{(m)}$ to measure how incomplete evidence affects the same encoder.
In the implemented experiments, raw modality encoders are frozen and all methods receive the same RoBERTa/CLIP or synthetic features \cite{roberta,radford2021clip,hamilton2017graphsage,velickovic2018gat}.
Few-shot adaptation, missing-modality completion, and graph-conditioned generation remain future-validation interfaces rather than completed evidence in this draft.
These definitions hand Method a fixed contract: spectral tokenization constructs $\Hh_b^{(m)}$, routing estimates $r_{i,b}^{(m)}$ from topology and availability, interaction and fusion produce $\Z_i$, and routed losses act on disjoint evidence roles.

\section{Method}
\label{sec:method}
\subsection{Overview}
\label{sec:method-overview}
SMGFM follows one design principle where graph-frequency roles are separated before multimodal semantics can interact.
This rule directly targets the two limitations identified in the Introduction.
For \textbf{Limitation One}, frequency-first tokenization prevents graph-smooth consensus and modality-private evidence from being mixed by a uniform spatial interaction layer.
For \textbf{Limitation Two}, routed supervision applies different objectives to different evidence roles instead of forcing every modality pair and frequency band into the same alignment target.

\begin{figure*}[t]
    \centering
    \includegraphics[width=\linewidth]{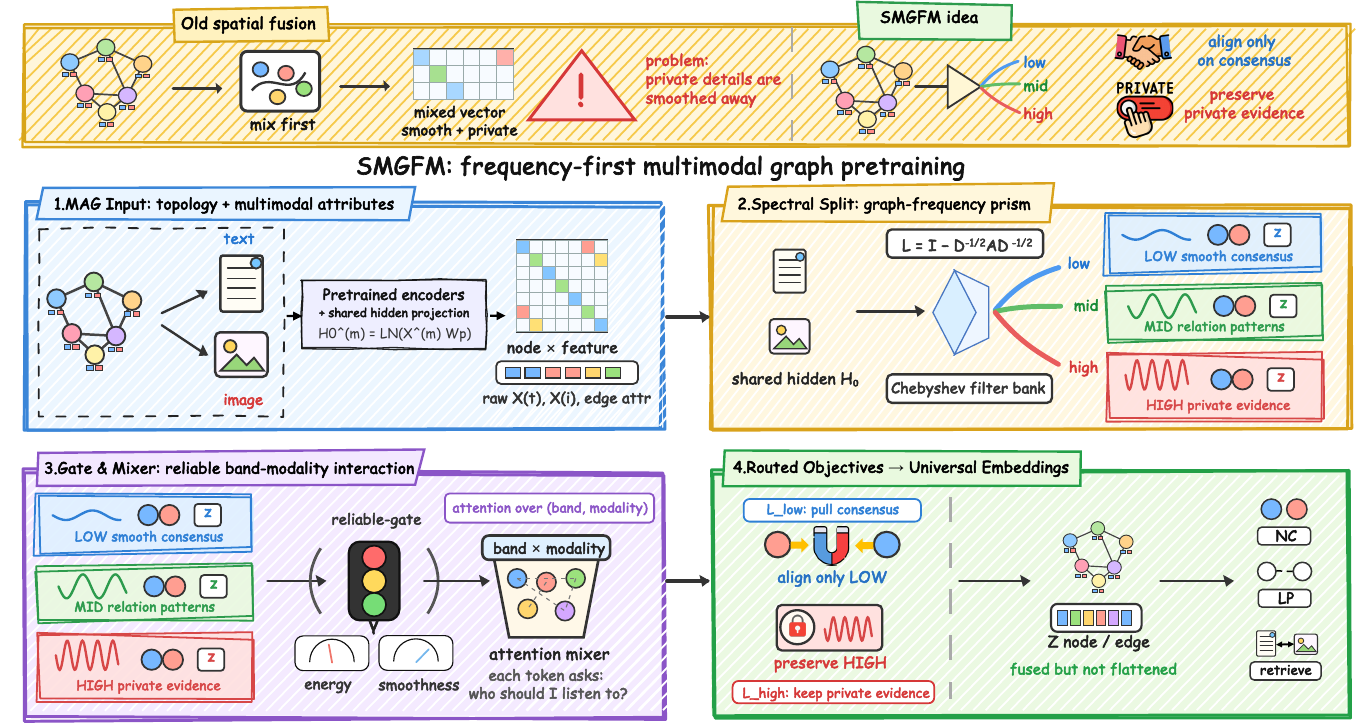}
    \caption{\textbf{SMGFM architecture.}
    SMGFM projects raw modalities to graph signals, constructs band-modality role tokens with Chebyshev filters, routes tokens by topology-conditioned reliability, performs role-constrained interaction and route fusion, and pretrains with frequency-routed objectives.}
    \label{fig:architecture}
\end{figure*}

As shown in Fig.~\ref{fig:architecture}, the method has three coupled modules.
\textbf{Frequency-first modality tokenization} converts each modality into low-, uncertain-, and high-frequency tokens.
\textbf{Topology-conditioned role routing and interaction} estimates which tokens should act as consensus, private, or uncertain evidence before cross-modal attention.
\textbf{Frequency-routed pretraining} aligns consensus routes, preserves private routes, separates their fused representations, and anchors the encoder to observed features and topology.
The result is an encoder that produces node representations $\Z_i$ while retaining modality-specific representations $\Z_i^{(m)}$ for modality-centric evaluation.

\subsection{Frequency-First Modality Tokenization via Chebyshev Filters}
\label{sec:freq-decomp}
\noindent\textbf{Modality projection.}
For a raw modality object $r_i^{(m)}$, SMGFM first obtains a frozen modality feature and projects it to a shared graph-signal space:
\begin{equation}
\X_i^{(m)}=\mathrm{Enc}_m(r_i^{(m)}),\quad
\Hh_{0,i}^{(m)}=\mathrm{LN}\!\left(\X_i^{(m)}\boldsymbol{W}_p^{(m)}+\boldsymbol{b}_p^{(m)}\right).
\label{eq:modality-projection}
\end{equation}
Here $\mathrm{Enc}_m$ may be RoBERTa, CLIP, or a synthetic encoder depending on the dataset, and unavailable modalities are represented by learned mask tokens with availability indicator $\rho_i^{(m)}$.
This step standardizes heterogeneous modality dimensions without making raw encoders part of the reported training budget.

\noindent\textbf{Polynomial filter bank.}
SMGFM then constructs a Chebyshev filter bank on the scaled Laplacian $\Lb$:
\begin{equation}
\begin{aligned}
\boldsymbol{T}_0^{(m)}&=\Hh_0^{(m)},\quad
\boldsymbol{T}_1^{(m)}=\Lb\Hh_0^{(m)},\\
\boldsymbol{T}_k^{(m)}&=2\Lb\boldsymbol{T}_{k-1}^{(m)}-\boldsymbol{T}_{k-2}^{(m)},\quad
2\leq k\leq K,\\
\Hh_b^{(m)}&=\sum_{k=0}^{K}\alpha_{b,k}\boldsymbol{T}_k^{(m)}.
\end{aligned}
\label{eq:band-filter}
\end{equation}
The coefficients $\alpha_{b,k}$ are learned per band $b$.
For the normalized Laplacian, we use $\lambda_{\max}\approx2$, so $\Lb=2\Lm/\lambda_{\max}-\I$ matches the Preliminaries definition and can be applied by sparse graph multiplications~\cite{defferrard2016chebnet}.
Thus the method obtains band tokens without eigendecomposition.
Bands are assigned to $\mathcal{B}_C$, $\mathcal{B}_P$, and $\mathcal{B}_U$: low-frequency tokens are consensus candidates, high-frequency tokens retain modality-private evidence, and uncertain middle bands are resolved by the router.

\subsection{Topology-Conditioned Role Routing}
\label{sec:role-routing}
\noindent\textbf{Routing cues.}
Tokenization exposes frequency roles, but graph frequency alone does not determine whether a token is reliable for a specific node.
For node $i$, band $b$, and modality $m$, SMGFM computes topology and modality cues:
\begin{equation}
\begin{gathered}
e_{i,b}^{(m)}
=
\|\Hh_{i,b}^{(m)}\|_2^2,
\\[2pt]
s_{i,b}^{(m)}
=
\sum_{j\in\mathcal{N}(i)}\A_{ij}
\|\Hh_{i,b}^{(m)}-\Hh_{j,b}^{(m)}\|_2^2,
\\[2pt]
\kappa_{i,b}^{(m)}
=
\frac{1}{|\M|-1}
\sum_{m'\neq m}\operatorname{cos}(\Hh_{i,b}^{(m)},\Hh_{i,b}^{(m')}),
\\[2pt]
\boldsymbol{c}_{i,b}^{(m)}
=
[\bar e_{i,b}^{(m)},\bar s_{i,b}^{(m)},
\rho_i^{(m)},\bar d_i,\bar\kappa_{i,b}^{(m)}],
\\[2pt]
r_{i,b}^{(m)}
=
\sigma\!\left(\psi_\theta(\boldsymbol{c}_{i,b}^{(m)})\right).
\end{gathered}
\label{eq:role-router}
\end{equation}
Here $e$ measures filter response, $s$ measures local smoothness, $\rho$ marks modality availability, $d_i$ is node degree, and $\kappa$ measures cross-modal agreement or conflict.
Bars denote normalized cues within a batch or sampled subgraph.
The score $r_{i,b}^{(m)}$ is a reliability route used before token interaction; it is not a generic modality weight.
Reported instantiations may use a cue subset when a dataset does not expose every cue, but the method principle is to decide token reliability before modalities are mixed.

\noindent\textbf{Why routing precedes interaction.}
Low-frequency agreement can be harmful when it comes from noisy popularity or label-irrelevant homophily, while high-frequency disagreement can be useful when it carries private text or image evidence.
Routing before interaction lets the model suppress unreliable consensus tokens, retain private tokens with strong local evidence, and defer ambiguous $\mathcal{B}_U$ tokens to the fusion layer instead of aligning them prematurely.

\subsection{Role-Constrained Band-Modality Interaction and Fusion}
\label{sec:cross-modal-attn}
\noindent\textbf{Band-modality interaction.}
For token $u=(b,m)$, define the reliability-weighted input
$\boldsymbol{x}_{i,u}=r_{i,b}^{(m)}\Hh_{i,b}^{(m)}$ and the role map $\tau(b)\in\{C,P,U\}$.
SMGFM performs attention over the $B|\M|$ tokens attached to node $i$:
\begin{equation}
\begin{aligned}
\boldsymbol{q}_{i,u}&=\boldsymbol{x}_{i,u}\boldsymbol{W}_Q,\quad
\boldsymbol{k}_{i,u}=\boldsymbol{x}_{i,u}\boldsymbol{W}_K,\quad
\boldsymbol{v}_{i,u}=\boldsymbol{x}_{i,u}\boldsymbol{W}_V,\\
\ell_{i,u,v}&=
\frac{\boldsymbol{q}_{i,u}\boldsymbol{k}_{i,v}^{\top}}{\sqrt{d_h}}
+\eta_{\tau(b),\tau(b')}+\mu_{m,m'}+\log(r_{i,b'}^{(m')}+\epsilon),\\
a_{i,u\to v}&=\softmax_v(\ell_{i,u,v}),\quad
\tilde{\boldsymbol{h}}_{i,u}=\sum_v a_{i,u\to v}\boldsymbol{v}_{i,v}.
\end{aligned}
\label{eq:role-interaction}
\end{equation}
Here $v=(b',m')$.
The role bias $\eta$ learns which role transitions are useful, the modality bias $\mu$ calibrates modality-pair compatibility, and the reliability term downweights tokens that the router judged unavailable or noisy.
Stacking $\tilde{\boldsymbol{h}}_{i,(b,m)}$ over nodes gives $\tilde{\Hh}_b^{(m)}$.
We reserve \textbf{interaction} for this attention over band-modality tokens.

\noindent\textbf{Route fusion.}
Fusion pools interacted tokens by route and then forms the task representation:
\begin{equation}
\begin{aligned}
\boldsymbol{z}_i^C&=\sum_{b\in\mathcal{B}_C}\sum_{m\in\M}
\omega_{i,b}^{(m)}\tilde{\boldsymbol{h}}_{i,(b,m)},\\
\boldsymbol{z}_i^P&=\sum_{b\in\mathcal{B}_P}\sum_{m\in\M}
\omega_{i,b}^{(m)}\tilde{\boldsymbol{h}}_{i,(b,m)},\\
\boldsymbol{z}_i^U&=\sum_{b\in\mathcal{B}_U}\sum_{m\in\M}
\omega_{i,b}^{(m)}\tilde{\boldsymbol{h}}_{i,(b,m)},\\
\Z_i&=\mathrm{MLP}\!\left(\boldsymbol{z}_i^C\Vert\boldsymbol{z}_i^P\Vert\boldsymbol{z}_i^U\right).
\end{aligned}
\label{eq:route-fusion}
\end{equation}
Here $\omega_{i,b}^{(m)}$ is a route-wise softmax.
Consensus fusion supports graph-smooth node and edge decisions; private fusion preserves local modality evidence for retrieval-sensitive and heterophily-sensitive behavior; uncertain fusion lets the model use mid-band evidence without assigning it a fixed role in advance.

\subsection{Frequency-Routed Pretraining}
\label{sec:pretraining}
\noindent\textbf{Consensus alignment.}
Consensus routes should agree across modalities only where graph frequency already suggests shared semantics.
Therefore SMGFM applies InfoNCE on $\mathcal{B}_C$ and does not globally align all tokens~\cite{oord2018representation,radford2021clip}:
\begin{equation}
\Lcal_C=
\sum_{b\in\mathcal{B}_C}\sum_{m\neq m'}
\mathrm{InfoNCE}(\Hh_b^{(m)},\Hh_b^{(m')}).
\label{eq:consensus-loss}
\end{equation}

\noindent\textbf{Private-evidence preservation.}
Private routes should not be collapsed by consensus alignment.
Let $\tilde{\Hh}_b^{(m)}$ be the post-interaction token and $\mathcal{E}_{\Lm}(\X)=\tr(\X^{\top}\Lm\X)$.
SMGFM preserves both Dirichlet energy and feature direction:
\begin{equation} 
\begin{aligned} 
\Lcal_P^{E} &=\sum_{b\in\mathcal{B}_P}\sum_{m\in\M} \left|\mathcal{E}_{\Lm}(\Hh_b^{(m)})-\mathcal{E}_{\Lm}(\tilde{\Hh}_b^{(m)})\right|,\\ \Lcal_P &=\Lcal_P^{E} +\beta_P\sum_{b\in\mathcal{B}_P}\sum_{m\in\M} \left(1-\operatorname{cos}(\mathrm{sg}(\Hh_b^{(m)}),\tilde{\Hh}_b^{(m)})\right). \end{aligned} 
\label{eq:private-preserve} 
\end{equation}
Eq.~\eqref{eq:private-preserve} separates the preservation objective into an energy-drift term and a direction-preservation term: $\Lcal_P^{E}$ controls graph-frequency drift, while the stop-gradient direction term prevents a high-frequency token from matching only in total energy.

\noindent\textbf{Role separation.}
To reduce re-entanglement after fusion, SMGFM penalizes consensus/private correlation:
\begin{equation}
\Lcal_R=
\left\|
\frac{(\Z^C)^{\top}\Z^P}{\|\Z^C\|_F\|\Z^P\|_F+\epsilon}
\right\|_F^2,\quad
\Z^\tau=[\boldsymbol{z}_1^\tau;\ldots;\boldsymbol{z}_{|\V|}^\tau].
\label{eq:role-separation}
\end{equation}

\noindent\textbf{Reconstruction anchors.}
Masked modality reconstruction $\Lcal_{\mathrm{mask}}$ predicts held-out modality features from $\Z_i$, and edge reconstruction $\Lcal_{\mathrm{edge}}$ scores sampled positive/negative edges from endpoint embeddings.
These anchors prevent role routing from becoming a purely contrastive shortcut and keep the encoder tied to both feature and topology evidence.

\noindent\textbf{Full objective and direct token-level routing.}
The final pretraining loss is
\begin{equation}
\Lcal_{\mathrm{core}}=
\Lcal_{\mathrm{mask}}+\lambda_C\Lcal_C+\lambda_P\Lcal_P
+\lambda_R\Lcal_R+\lambda_E\Lcal_{\mathrm{edge}}.
\label{eq:total-loss}
\end{equation}
Because $\Lcal_C$ and $\Lcal_P$ are applied to disjoint band supports,
\begin{equation}
\begin{aligned}
\frac{\partial \Lcal_C}{\partial \Hh_b^{(m)}}&=0,\quad b\notin\mathcal{B}_C,\\
\frac{\partial \Lcal_P}{\partial \Hh_b^{(m)}}&=0,\quad b\notin\mathcal{B}_P,\quad
\mathcal{B}_C\cap\mathcal{B}_P=\varnothing.
\end{aligned}
\label{eq:direct-routing}
\end{equation}
This is direct token-level routing of losses, not full gradient independence: projections, filters, router, mixer, and fusion parameters remain shared and are trained end to end.

\begin{algorithm}[!ht]
    \DontPrintSemicolon
    \SetAlgoCaptionSeparator{.}
    \scriptsize
    \caption{$\kw{SMGFM Pretrain}(\G,\{\X^{(m)}\}_{m\in\M})$}
    \label{alg:smgfm-pretrain}
    \textbf{Input.} MAG $\G=(\V,\E,\A)$, modality features $\{\X^{(m)}\}$, $K,B,p_m$, and $\lambda_C,\lambda_P,\lambda_R,\lambda_E$\;
    \textbf{Output.} Pretrained encoder $f_\theta$\;
    Initialize projections, Chebyshev filters, router, mixer, fusion module, and reconstruction heads; construct $\Lm$ and $\Lb$\;
    \For{each training epoch}{
        Sample $\G_s$, modality masks, edge pairs, and low-route modality pairs\;
        Project each available modality with Eq.~\eqref{eq:modality-projection}; replace missing modalities by mask tokens\;
        Build Chebyshev bases and band tokens with Eq.~\eqref{eq:band-filter}\;
        Compute routing scores $r_{i,b}^{(m)}$ with Eq.~\eqref{eq:role-router}\;
        Interact tokens and fuse routes with Eqs.~\eqref{eq:role-interaction}--\eqref{eq:route-fusion}\;
        Compute Eq.~\eqref{eq:total-loss} and update $\theta$\;
    }
    \Return{$f_\theta$}\;
\end{algorithm}

\subsection{Downstream Adaptation and Reported Instantiations}
\label{sec:downstream-adaptation}
After pretraining, SMGFM uses lightweight heads rather than retraining raw modality encoders.
Node classification applies an MLP probe to $\Z_i$, link prediction scores endpoint pairs from fused embeddings, node clustering runs k-means on frozen $\Z_i$, and retrieval compares modality-specific representations $\Z_i^{(m)}$.
Controlled mechanism tests instantiate the neural pipeline above.
For the public real-data subsets, \textbf{SMGFM-Expert} is reported as a matched-protocol spectral expert bank: it builds raw, low-pass, high-pass, and cross-modal agreement/disagreement experts, then selects or gates experts on validation data for each local task.
This instantiation provides frequency-aware evidence under the matched local protocol; it is not presented as a new full benchmark-scale pretraining run.

\begin{table*}[t]
\centering
\setlength{\abovecaptionskip}{0.2cm}
\caption{Experimental dataset statistics and evidence roles. Controlled MAG supports mechanism diagnostics, public real-data subsets support matched graph-task validation after local 6K-node preprocessing, and Goodreads-Cover is a pilot sanity check.}
\footnotesize
\label{tab:dataset-details}
\renewcommand{\arraystretch}{1.18}
\resizebox{0.9\textwidth}{!}{
\setlength{\tabcolsep}{2mm}{
\begin{tabular}{ccccccc}
\toprule
Dataset & \#Nodes & \#Edges & Modality Features & Train/Val/Test & Task & Description \\
\midrule[0.3pt]
Controlled MAG & 500 & generated & 2$\times$128 & 60\%/20\%/20\% & Diagnostic & synthetic mechanism test \\
\midrule[0.3pt]
Movies & 6,000 & 21,216 & RoBERTa, CLIP & 60\%/20\%/20\% & Transductive & product graph subset \\
Grocery & 6,000 & 17,925 & RoBERTa, CLIP & 60\%/20\%/20\% & Transductive & product graph subset \\
Toys & 6,000 & 9,500 & RoBERTa, CLIP & 60\%/20\%/20\% & Transductive & product graph subset \\
\midrule[0.3pt]
Reddit-S & 6,000 & 41,413 & RoBERTa, CLIP & 60\%/20\%/20\% & Transductive & social graph subset \\
Reddit-M & 6,000 & 2,220 & RoBERTa, CLIP & 60\%/20\%/20\% & Transductive & social graph subset \\
\midrule[0.3pt]
Goodreads-Cover & 600 & 542 directed & metadata, cover & 60\%/20\%/20\% & Pilot & book graph sanity check \\
\bottomrule
\end{tabular}
}}
\end{table*}

\begin{table}[t]
\centering
\caption{Fine-tuning and evaluation protocols after SMGFM pretraining. Raw modality encoders stay frozen, and only task heads, light projections, or clustering assignments are fitted.}
\label{tab:finetune}
\footnotesize
\setlength{\tabcolsep}{2pt}
\renewcommand{\arraystretch}{1.18}
\begin{tabularx}{\linewidth}{@{}>{\raggedright\arraybackslash}p{0.22\linewidth}>{\centering\arraybackslash}p{0.14\linewidth}YY@{}}
\toprule
Task & Input & Trainable and frozen parts & Metrics and split \\
\midrule
Node classification & $\Z_i$ & MLP probe or adapter; raw encoders frozen & Acc., Macro-F1 on train/val/test split \\
Link prediction & $(\Z_i,\Z_j)$ & Edge decoder; raw encoders frozen & MRR, Hits@3, AUC, AP on edge split \\
Node clustering & $\Z_i$ & k-means only; all encoder modules frozen & NMI, ARI with labels for evaluation only \\
Retrieval & $\Z_i^{t},\Z_i^{v}$ & No head or light projection; graph encoder frozen & R@K, MRR on paired nodes \\
\bottomrule
\end{tabularx}
\end{table}

\subsection{Theoretical Analysis and Complexity}
\label{sec:theoretical-analysis}
We give sufficient-condition support for low-band alignment, high-band preservation, and reliability-weighted fusion; this is not a global optimality or convergence proof.
Let $\Lm=\Ut\Lambdab\Ut^{\top}$ and define
\begin{equation}
\begin{aligned}
\widehat{\X}&=\Ut^{\top}\X,\quad
\mathcal{E}_{\Lm}(\X)=\sum_r\lambda_r\|\widehat{\X}_r\|_2^2,\\
g_b(\lambda)&=\sum_{k=0}^{K}\alpha_{b,k}T_k(2\lambda/\lambda_{\max}-1).
\end{aligned}
\label{eq:theory-setup}
\end{equation}
Assume the diagnostic model $\Hh_0^{(m)}=\boldsymbol{S}+\boldsymbol{P}^{(m)}+\boldsymbol{N}^{(m)}$, where $\boldsymbol{S}$ is low-frequency shared evidence, $\boldsymbol{P}^{(m)}$ is modality-private evidence, and $\boldsymbol{N}^{(m)}$ is residual noisy evidence.

\begin{proposition}[Consensus leakage control]
For $b\in\mathcal{B}_C$, let $\epsilon_b^{H}=\sup_{\lambda\in[\lambda_c,2]}|g_b(\lambda)|$.
If $\boldsymbol{P}^{(m)}+\boldsymbol{N}^{(m)}$ is supported on $[\lambda_c,2]$, then $\|g_b(\Lm)(\boldsymbol{P}^{(m)}+\boldsymbol{N}^{(m)})\|_2\leq\epsilon_b^{H}\|\boldsymbol{P}^{(m)}+\boldsymbol{N}^{(m)}\|_2$.
\end{proposition}
\noindent\textit{Proof sketch and takeaway.}
This formalizes the intuition that low-band alignment is reliable only when the selected filter has small high-frequency gain.
By the spectral theorem, $g_b(\Lm)=\Ut g_b(\Lambdab)\Ut^{\top}$ has induced norm at most $\epsilon_b^H$ on $[\lambda_c,2]$, so $\Lcal_C$ is well targeted only when non-consensus leakage is suppressed, as Q2 probes through low-band agreement.

\begin{proposition}[Private-route preservation and disjoint supervision]
For $b\in\mathcal{B}_P$, let $\epsilon_b^{L}=\sup_{\lambda\in[0,\lambda_c]}|g_b(\lambda)|$.
If $\boldsymbol{S}$ is supported on $[0,\lambda_c]$, then $\|g_b(\Lm)\boldsymbol{S}\|_2\leq\epsilon_b^{L}\|\boldsymbol{S}\|_2$.
Moreover, when $\mathcal{B}_C\cap\mathcal{B}_P=\varnothing$, Eq.~\eqref{eq:direct-routing} applies $\Lcal_C$ and $\Lcal_P$ to disjoint band tokens while shared modules remain end-to-end trainable.
\end{proposition}
\noindent\textit{Proof sketch and takeaway.}
The same spectral-norm argument bounds smooth leakage, while Eq.~\eqref{eq:direct-routing} gives zero direct loss derivatives outside each assigned band set.
Thus high-band preservation prevents private evidence from collapsing into smooth consensus; when Eq.~\eqref{eq:private-preserve} is small, its energy and direction terms limit high-route drift, motivating the Q2/Q3 NoPres diagnostics.

\begin{proposition}[Reliability-weighted fusion]
Let each interacted token estimate a target route signal as $\boldsymbol{y}_u=\boldsymbol{y}^{\star}+\boldsymbol{\xi}_u$, where token errors are zero-mean, mutually uncorrelated, and $\mathbb{E}\|\boldsymbol{\xi}_u\|_2^2\leq\sigma_u^2$.
For nonnegative fusion weights $\sum_u\omega_u=1$, $\mathbb{E}\|\sum_u\omega_u\boldsymbol{y}_u-\boldsymbol{y}^{\star}\|_2^2\leq\sum_u\omega_u^2\sigma_u^2$.
\end{proposition}
\noindent\textit{Proof sketch and takeaway.}
Expanding the squared error cancels cross terms by zero mean and uncorrelated errors.
Reliability-weighted fusion is useful only when larger routing weights favor lower-error tokens; Q4 treats this correlation as a boundary condition, not an arbitrary-shift claim.

\noindent\textbf{Complexity.}
Let $n=|\V|$, $e=|\E|$, $M=|\M|$, $d$ be hidden width, $L$ propagation layers, $K$ Chebyshev order, and $B$ frequency bands.
Let $\gamma=BM$ denote the number of band-modality routes per node.
Because the efficiency question concerns multimodal graph encoders, Table~\ref{tab:complexity} compares multimodal GNN baselines and the spatial-fusion control used in Sec.~\ref{sec:experimental-setups}.
\begin{table}[H]
\centering
\setlength{\abovecaptionskip}{1pt}
\setlength{\belowcaptionskip}{-4pt}
\caption{Complexity of multimodal graph encoders.}
\label{tab:complexity}
\footnotesize
\setlength{\tabcolsep}{1.6pt}
\renewcommand{\arraystretch}{0.9}
\resizebox{0.9\columnwidth}{!}{
\begin{tabular}{@{}lll@{}}
\toprule
Method & Inference Efficiency & Memory Size \\
\midrule
MMGCN & $\mathcal{O}(ML(ed+nd^2))$ & $\mathcal{O}(MLnd+e)$ \\
MGAT & $\mathcal{O}(MLed)$ & $\mathcal{O}(ML(nd+e))$ \\
SMGFM-Spatial & $\mathcal{O}(Med+nM^2d)$ & $\mathcal{O}(Mnd+e)$ \\
SMGFM & $\mathbf{\mathcal{O}(MKed+MBKnd+n\gamma^2d)}$ & $\mathbf{\mathcal{O}(\gamma nd+e)}$ \\
\bottomrule
\end{tabular}}
\end{table}
The table should be read as a route-level accounting rather than a claim that SMGFM is cost-free relative to spatial fusion.
The term $MKed$ comes from applying $K$-order sparse Chebyshev recursions to each modality, and $MBKnd$ accounts for band-wise projection and route construction.
The only quadratic term is $n\gamma^2d$: it compares the $\gamma=BM$ band-modality routes attached to the same node.
Thus, when $B$, $K$, and $M$ are fixed design constants, SMGFM scales as $\mathcal{O}(e+n)$ with graph size under feature inference process, while avoiding dense node-pair interaction such as $\mathcal{O}(n^2d)$.

This comparison also explains where SMGFM spends computation.
Compared with MMGCN and MGAT, it replaces repeated modality-wise graph propagation or edge interaction with sparse spectral tokenization followed by local route interaction.
Compared with SMGFM-Spatial, it pays explicit $B$ and $K$ factors so that graph-smooth consensus and modality-private evidence are separated before interaction.
This is the intended trade-off: the additional route budget is not an incidental overhead, but the mechanism that makes frequency-routed alignment and preservation testable in Q2 and Q3.

Training follows the same forward-order terms plus lightweight objective costs on selected routes.
Low-band alignment acts only on $\mathcal{B}_C$ tokens, high-band preservation acts only on $\mathcal{B}_P$ tokens, and reconstruction uses the retained route representations; these losses add route-linear costs dominated by tokenization and interaction when $B$ and $M$ are small.
The memory expression $\mathcal{O}(\gamma nd+e)$ stores band-modality tokens and sparse graph structure, while mini-batch or subgraph execution replaces $n,e$ with the active batch statistics.
Consequently, the efficiency claim is a complexity-level scalability argument under fixed features and controlled route counts, not a measured speedup claim.

\section{Experiments}
In this section, we provide the detailed experimental evaluation on SMGFM using the publicly available datasets. 
This section answers four questions:
\textbf{Q1}: Does SMGFM provide matched graph-task utility and controlled retrieval-sensitive binding?
\textbf{Q2}: Do spectral diagnostics show that routed objectives affect their intended frequency roles?
\textbf{Q3}: Which modules and hyperparameters are necessary?
\textbf{Q4}: What boundaries appear under missing modalities, heterophily, and efficiency constraints?

\subsection{Experimental Setups}
\label{sec:experimental-setups}
\textbf{Datasets.}
We evaluate SMGFM on real-world multimodal-attributed graph datasets and the detailed statistics of datasets is presented in Table~ref{tab:dataset-details}. Specifically, the controlled datasets refers on several table refers to those that are designed to assess whether the model can distinguish graph-smooth semantics from modality-specific high-frequency information under heterophily and missing-modality settings~\cite{2020h2gcn,ma2021smil,ma2022missingtransformer,wang2023shaspec}. For real-world evaluation, we use public multimodal graph subsets, include Movies, Grocery, Toys, Reddit-S, and Reddit-M, where node attributes are represented by RoBERTa and CLIP features~\cite{zhu2025mmgraph,yan2024multimodalbenchmark,wan2026openmag,roberta,radford2021clip}. All datasets used in this study are publicly available, and the implementation of SMGFM is released at \url{https://anonymous.4open.science/r/SMGFM-4749}.

\textbf{Baselines and probes.}
Controlled variants isolate mechanisms: \textbf{SMGFM-Spatial} removes frequency tokenization, \textbf{SMGFM-NoAlign} removes low-band alignment, \textbf{SMGFM-NoPres} removes high-band preservation, \textbf{SMGFM-Uni} removes cross-modal binding, and \textbf{SMGFM-Global} uses global alignment.

\subsection{Performance Comparison}
\label{sec:rq1}
To answer \textbf{Q1}, Table~\ref{tab:realdata-baselines} and Fig.~\ref{fig:synthetic} separate matched graph-task utility from controlled cross-modal binding.

\textbf{Matched real-data graph tasks.}
Table~\ref{tab:realdata-baselines} shows a consistent matched-protocol pattern: under the same experimental settings, SMGFM-Expert takes the leading position on the reported product-graph node-classification tasks and gives the clearest advantage on the Grocery link-prediction task.
This pattern is important because the comparison does not change the input encoders or evaluation protocol; the difference comes from how spectral experts expose graph-smooth and modality-dependent evidence to the downstream probe.
The clustering columns add a useful boundary rather than a contradiction.
When frozen RoBERTa/CLIP features already form strong semantic clusters, SMGFM preserves that structure and ties the strongest raw-feature behavior instead of forcing an unnecessary spectral reshaping.
Thus Table~\ref{tab:realdata-baselines} supports matched real-data graph-task utility under the reported local protocol, not a full benchmark leaderboard claim.

\textbf{Controlled retrieval-sensitive binding.}
Fig.~\ref{fig:synthetic} gives the mechanism counterpart: on the clean controlled graph, SMGFM clearly separates from spatial fusion and the unimodal variant on text-image retrieval.
The separation indicates that retrieval-sensitive binding benefits when topology and modalities are routed before fusion, because cross-modal evidence can interact after its graph-frequency role has been identified.
The boundary is equally important: spatial smoothing remains stronger for highly graph-smooth node/link probes, so Q1 supports controlled cross-modal binding rather than dominance on every smooth-label task.

\begin{table*}[t]
\centering
\caption{Performance Comparison on Real-Data Multimodal Datasets. Bold indicate the \textbf{Best Performance}.}
\label{tab:realdata-baselines}
\footnotesize
\setlength{\tabcolsep}{2pt}
\renewcommand{\arraystretch}{1.10}
\resizebox{\textwidth}{!}{
\begin{tabular}{lcccccccccc}
\toprule
Method & Movies Acc. & Movies F1 & Grocery Acc. & Grocery F1 & Movies MRR & Movies H@3 & Grocery MRR & Grocery H@3 & Toys NMI/ARI & Reddit-M NMI/ARI \\
\midrule
MLP & $46.11_{\scriptstyle\pm0.39}$ & $39.59_{\scriptstyle\pm1.85}$ & $74.97_{\scriptstyle\pm0.26}$ & $66.71_{\scriptstyle\pm0.25}$ & $26.71_{\scriptstyle\pm3.50}$ & $29.65_{\scriptstyle\pm1.23}$ & $15.90_{\scriptstyle\pm2.22}$ & $21.65_{\scriptstyle\pm1.56}$ & $\mathbf{28.76}_{\scriptstyle\pm0.19}/\mathbf{14.28}_{\scriptstyle\pm0.20}$ & $\mathbf{65.26}_{\scriptstyle\pm0.31}/\mathbf{42.19}_{\scriptstyle\pm1.30}$ \\
GPR-GNN & $47.69_{\scriptstyle\pm1.01}$ & $42.93_{\scriptstyle\pm1.38}$ & $77.47_{\scriptstyle\pm0.41}$ & $69.98_{\scriptstyle\pm1.14}$ & $61.34_{\scriptstyle\pm6.04}$ & $82.41_{\scriptstyle\pm4.59}$ & $50.02_{\scriptstyle\pm4.49}$ & $71.87_{\scriptstyle\pm4.64}$ & $27.25_{\scriptstyle\pm0.41}/5.35_{\scriptstyle\pm0.24}$ & $47.24_{\scriptstyle\pm0.70}/15.65_{\scriptstyle\pm0.45}$ \\
SMGFM-Expert & $\mathbf{49.72}_{\scriptstyle\pm0.99}$ & $\mathbf{44.08}_{\scriptstyle\pm0.83}$ & $\mathbf{78.22}_{\scriptstyle\pm0.64}$ & $\mathbf{70.03}_{\scriptstyle\pm1.81}$ & $\mathbf{62.42}_{\scriptstyle\pm10.00}$ & $\mathbf{88.96}_{\scriptstyle\pm4.24}$ & $\mathbf{74.77}_{\scriptstyle\pm14.94}$ & $\mathbf{77.40}_{\scriptstyle\pm15.62}$ & $\mathbf{28.76}_{\scriptstyle\pm0.19}/\mathbf{14.28}_{\scriptstyle\pm0.20}$ & $\mathbf{65.26}_{\scriptstyle\pm0.31}/\mathbf{42.19}_{\scriptstyle\pm1.30}$ \\
\bottomrule
\end{tabular}}
\end{table*}

\begin{figure}[!t]
    \centering
    \includegraphics[width=\linewidth]{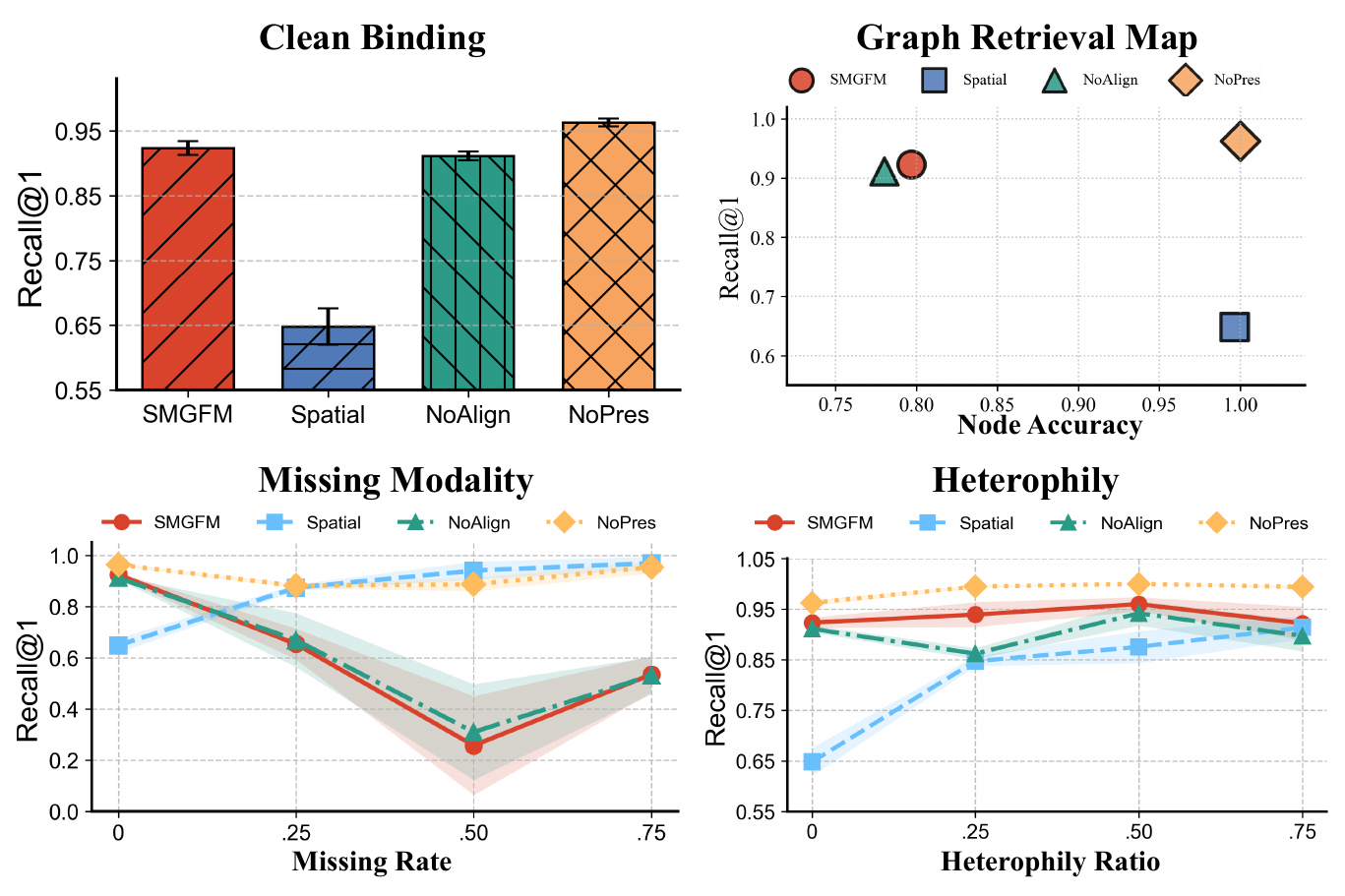}
    \caption{\textbf{Controlled benchmark and stress tests.}
    Clean binding shows higher text-image retrieval than spatial fusion, the graph-retrieval map shows the trade-off against smooth-label accuracy, missing-modality stress exposes a current weakness, and heterophily stress shows stable retrieval under edge conflict.}
    \label{fig:synthetic}
\end{figure}

\subsection{Spectral Diagnostic Analysis}
\label{sec:spectral-diagnostics}
To answer \textbf{Q2}, Fig.~\ref{fig:mechanism} and Table~\ref{tab:retrieval-stability} test whether routed objectives leave the intended spectral fingerprints.
Low-band cosine measures smooth-channel agreement; high-band drift measures private-frequency energy change after interaction and fusion.

\textbf{Low-band alignment.}
Removing low-band alignment nearly erases the smooth-channel agreement pattern in Fig.~\ref{fig:mechanism}.
Because $\Lcal_C$ acts only on low-frequency tokens, the diagnostic supports graph-smooth consensus alignment rather than generic representation similarity.

\textbf{High-band preservation.}
The preservation diagnostic explains why the strongest clean retrieval entry is not automatically the best mechanism.
NoPres obtains the most aggressive retrieval behavior in the clean setting, but Table~\ref{tab:retrieval-stability} shows that it also produces the largest high-band drift.
SMGFM keeps retrieval strong while avoiding that drift-heavy shortcut, which is the behavior expected when private high-frequency evidence is preserved instead of rewritten during interaction and fusion.
Overall, Q2 supports separate roles for low-band consensus and high-band private evidence, but remains a mechanism diagnostic rather than a robustness guarantee.
\begin{figure}[H]
    \centering
    \includegraphics[width=\linewidth]{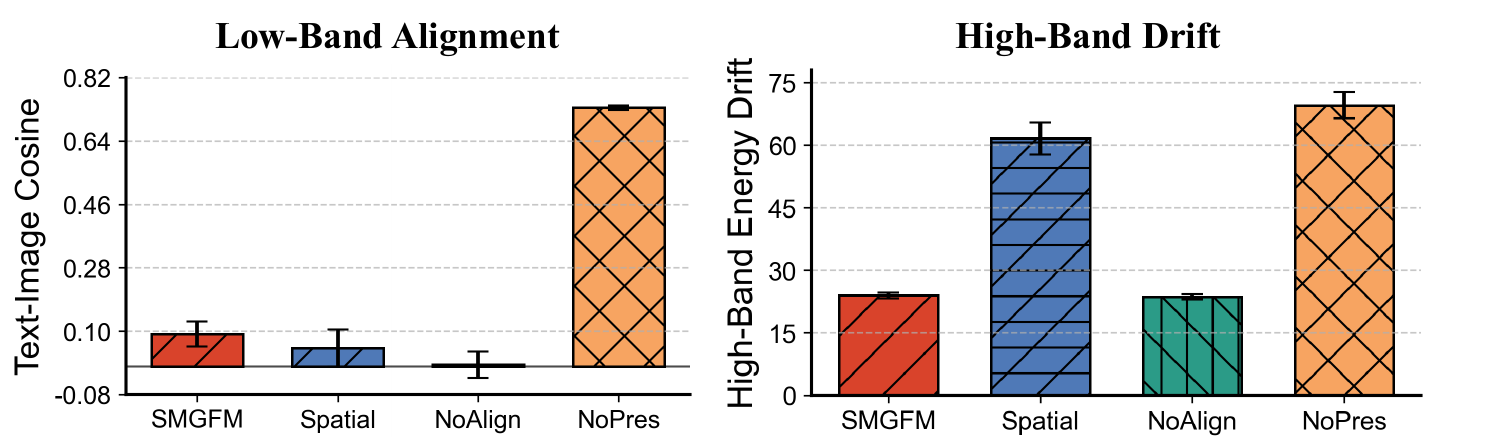}
    \caption{\textbf{Frequency-mechanism diagnostics.}
    Low-band alignment reports paired text-image cosine in the smooth channel, while high-band drift reports the change in private high-frequency Dirichlet energy after fusion.}
    \label{fig:mechanism}
\end{figure}

\begin{table}[H]
\centering
\caption{Retrieval and high-band stability on the controlled clean setting. The strongest clean-retrieval variant also has the largest drift.}
\label{tab:retrieval-stability}
\footnotesize
\setlength{\tabcolsep}{3pt}
\renewcommand{\arraystretch}{1.08}
\begin{tabularx}{\linewidth}{lccY}
\toprule
Variant & Recall@1 & Drift & Interpretation \\
\midrule
SMGFM & $0.923_{\scriptstyle \pm 0.01}$ & $24.0_{\scriptstyle \pm 0.7}$ & balanced retrieval with controlled drift \\
Spatial & $0.648_{\scriptstyle \pm 0.03}$ & $61.6_{\scriptstyle \pm 3.8}$ & weak binding after spatial fusion \\
NoAlign & $0.911_{\scriptstyle \pm 0.01}$ & $23.6_{\scriptstyle \pm 0.7}$ & stable but weak low-band consensus \\
NoPres & $0.963_{\scriptstyle \pm 0.01}$ & $69.6_{\scriptstyle \pm 3.2}$ & highest retrieval with largest drift \\
\bottomrule
\end{tabularx}
\end{table}

\subsection{Ablation and Sensitivity Analysis}
\label{sec:rq3}
To answer \textbf{Q3}, Fig.~\ref{fig:ablation} tests whether each module controls the evidence role assigned to it, and whether sensitivity curves reveal stability trade-offs rather than monotonic gains.

\textbf{Core ablations.}
The core ablations follow the expected method ordering.
Removing spectral decomposition makes the model behave much closer to the spatial-fusion variant, identifying frequency-first tokenization as the main binding control.
NoAlign primarily weakens the consensus diagnostic, whereas NoPres primarily weakens the stability diagnostic.
This division matters because each failure mode maps back to a specific design choice: low-band supervision organizes graph-smooth agreement, while high-band preservation prevents private evidence from being overwritten.
SMGFM-Global shows that broad alignment can still produce strong retrieval, but routed supervision is more interpretable because each diagnostic maps back to a specific frequency role.

\textbf{Sensitivity.}
The sensitivity panels show that more spectral capacity is not automatically better.
Using more bands fragments the controlled graph signal and raises memory use, indicating that a small graph can become harder to route when spectral capacity is over-partitioned.
Changing Chebyshev order and loss weights similarly moves retrieval, agreement, and drift in different directions rather than along a single dominant axis.
The useful pattern is a retrieval-stability balance, not a monotonic capacity story.

\begin{figure}[H]
    \centering
    \includegraphics[width=\linewidth]{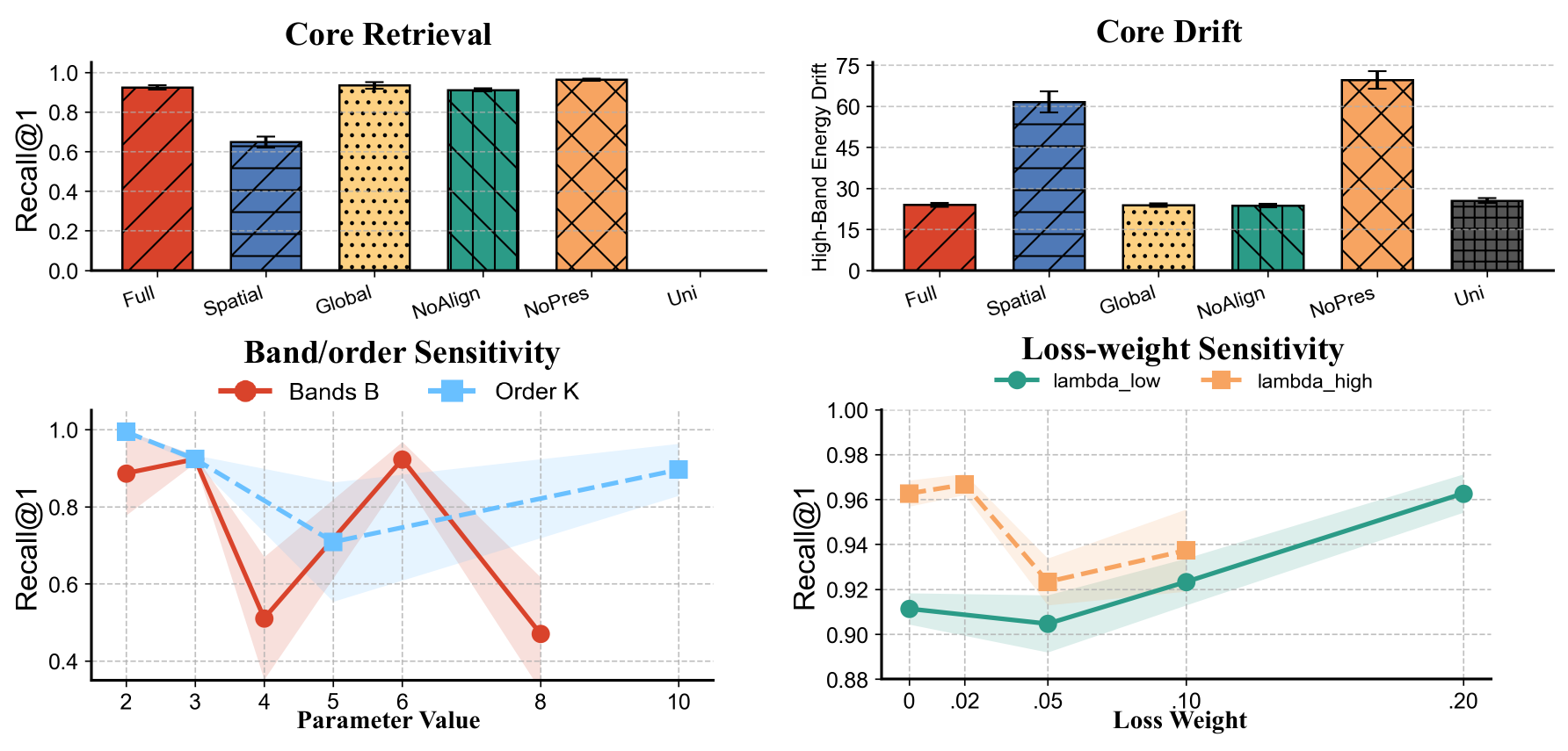}
    \caption{\textbf{Ablation and sensitivity suite.}
    Core panels test spectral decomposition, low-band alignment, and high-band preservation; sensitivity panels show retrieval-stability trade-offs.}
    \label{fig:ablation}
\end{figure}

\subsection{Stress and Efficiency Boundaries}
\label{sec:rq4}
To answer \textbf{Q4}, we treat missing modalities, heterophily, and efficiency as boundary checks, not broad deployment or robustness claims.

\textbf{Missing modalities.}
The missing-modality test exposes a current weakness: retrieval degrades sharply as paired evidence is removed.
This suggests that routed objectives need enough paired modality evidence to form compatible consensus and private routes; future completion should add frequency-aware imputation or stronger modality dropout.

\textbf{Heterophily.}
Under controlled heterophily, SMGFM stays stable across the tested edge-conflict settings, while SMGFM-Spatial remains lower at matched settings.
This supports controlled retrieval stability under edge conflict, not full spectral invariance across graph regimes.

\textbf{Efficiency, pilot boundary, and reproducibility.}
SMGFM uses $B$ Chebyshev filters of order $K$, giving $O(BK|\E|d)$ sparse multiplication cost per modality and avoiding $O(|\V|^3)$ eigendecomposition.
Because runs use uniform modality features and lightweight probes, the evidence supports complexity-level efficiency rather than measured speedup.
The strongest supported conclusion is mechanism-centered: SMGFM shows controlled binding and matched graph-task utility under the reported evidence chain, while missing evidence and full-scale validation remain boundaries.

\section{Conclusion}
In this paper, we identify frequency-blind mixing as a central obstacle for Multimodal-Attributed Graph (MAG) pretraining, where graph-smooth consensus and modality-private evidence can be entangled before the encoder decides how each signal should be used.
To address this issue, we propose \textbf{SMGFM}, a spectral multimodal graph pretraining framework that performs frequency-first modality tokenization, topology-guided spectral fusion, and frequency-routed pretraining.
SMGFM aligns low-frequency bands to capture graph-smooth cross-modal consensus and preserves high-frequency bands to retain modality-private evidence, making the roles of alignment and preservation directly testable.
Experiments show stronger controlled text-image binding than spatial fusion, matched-protocol gains on public real-data node classification and link prediction, and diagnostic evidence from low-band alignment, high-band drift, and ablation analyses.
Overall, SMGFM offers a mechanism-centered step toward frequency-aware multimodal graph pretraining, where topology and multimodal semantics are routed according to their graph-frequency roles.

\bibliographystyle{IEEEtran}
\bibliography{references}

\end{document}